%% file: paper_icml.tex
\documentclass{article}

\usepackage{microtype}
\usepackage{graphicx}
\usepackage{subfigure}
\usepackage{booktabs} 

\usepackage{hyperref}

\usepackage[accepted]{icml2018}

\icmltitlerunning{Dissecting Adam: The Sign, Magnitude and Variance of Stochastic Gradients}

\usepackage{tikz}
\usetikzlibrary{shapes,arrows,positioning}
\definecolor{mpg}{rgb}{0, 0.4717, 0.4604}
\definecolor{blu}{rgb}{0., 0., 0.509}

\usepackage{amsmath, amssymb}
\usepackage{amsthm}
\newtheorem{theorem}{Theorem}
\newtheorem{lemma}{Lemma}
\newtheorem*{lemma*}{Lemma}

\newtheorem*{problem*}{Model Problem}
\newtheorem{assumption}{Assumption}

\usepackage{soul}
\usepackage{color}

\usepackage{tabularx}

\DeclareMathOperator{\sign}{sign}
\DeclareMathOperator{\tr}{tr}
\newcommand{\erf}{\operatorname{erf}}
\newcommand{\B}{\mathcal{B}}
\renewcommand{\L}{\mathcal{L}}
\DeclareMathOperator*{\argmin}{arg\,min}

\usepackage{graphicx}
\usepackage{placeins}

\begin{document}

\twocolumn[
\icmltitle{Dissecting Adam: The Sign, Magnitude and Variance of Stochastic Gradients}



\icmlsetsymbol{equal}{*}

\begin{icmlauthorlist}
\icmlauthor{Lukas Balles}{mpi}
\icmlauthor{Philipp Hennig}{mpi}
\end{icmlauthorlist}

\icmlaffiliation{mpi}{Max Planck Institute for Intelligent Systems, T\"ubingen, Germany}

\icmlcorrespondingauthor{Lukas Balles}{lballes@tue.mpg.de}

\icmlkeywords{Optimization, Stochastic Optimization, Adam, Deep Learning, Machine Learning, ICML}

\vskip 0.3in
]



\printAffiliationsAndNotice{}  

\begin{abstract}
The \textsc{adam} optimizer is exceedingly popular in the deep learning community.
Often it works very well, sometimes it doesn't. Why?
We interpret \textsc{adam} as a combination of two aspects: for each weight, the update direction is determined by the \emph{sign} of stochastic gradients, whereas the update magnitude is determined by an estimate of their \emph{relative variance}.
We disentangle these two aspects and analyze them in isolation, gaining insight into the mechanisms underlying \textsc{adam}.
This analysis also extends recent results on adverse effects of \textsc{adam} on generalization, isolating the sign aspect as the problematic one.
Transferring the variance adaptation to \textsc{sgd} gives rise to a novel method, completing the practitioner's toolbox for problems where \textsc{adam} fails.
\end{abstract}

\setcounter{footnote}{1}

\section{Introduction}

Many prominent machine learning models pose empirical risk minimization problems with objectives of the form
\begin{align}
\label{eq:erm_problem}
\L (\theta) &= \frac{1}{M} \sum_{k=1}^M \ell(\theta; x_k),\\
\nabla \L(\theta) & = \frac{1}{M} \sum_{k=1}^M \nabla \ell(\theta; x_k),
\end{align}
where $\theta\in\mathbb{R}^d$ is a vector of parameters, $\{x_1, \dotsc, x_M\}$ is a training set, and $\ell(\theta; x)$ is a loss quantifying the performance of parameters $\theta$ on example $x$.
Computing the exact gradient in each step of an iterative optimization algorithm becomes inefficient for large $M$.
Instead, we sample a mini-batch $\B \subset \{1,\dotsc, M\}$ of size $\vert \B\vert \ll M$ with data points drawn uniformly and independently from the training set and compute an approximate \emph{stochastic gradient}
\begin{equation}
\label{eq:stochastic_gradient}
g(\theta) = \frac{1}{\vert \B\vert} \sum_{k\in\B} \nabla \ell(\theta; x_k),
\end{equation}
which is a random variable with $\mathbf{E}[g(\theta)]=\nabla \L(\theta)$.
An important quantity for this paper will be the (element-wise) variances of the stochastic gradient, which we denote by $\sigma_i^2(\theta) := \mathbf{var}[g(\theta)_i]$.


Widely-used stochastic optimization algorithms are stochastic gradient descent \citep[\textsc{sgd},][]{Robbins1951} and its momentum variants \citep{Polyak1964, Nesterov1983}.
A number of methods popular in deep learning choose per-element update magnitudes based on past gradient observations.
Among these are \textsc{adagrad}~\citep{Duchi2011}, \textsc{rmsprop}~\citep{Tieleman2012}, \textsc{adadelta}~\citep{Zeiler2012}, and \textsc{adam}~\citep{Kingma2014}.

\emph{Notation:} In the following, we occasionally drop $\theta$, writing $g$ instead of $g(\theta)$, et cetera.
We use shorthands like $\nabla\L_t$, $g_t$ for sequences $\theta_t$ and double indices where needed, e.g., $g_{t,i} = g(\theta_t)_i$, $\sigma^2_{t,i}= \sigma_i^2(\theta_t)$.
Divisions, squares and square-roots on vectors are to be understood \emph{element-wise}. To avoid confusion with inner products, we explicitly denote element-wise multiplication of vectors by $\odot$.

\subsection{A New Perspective on Adam}

\label{new_perspective_on_adam}

We start out from a reinterpretation of the widely-used \textsc{adam} optimizer,\footnote{Some of our considerations naturally extend to \textsc{adam}'s relatives \textsc{rmsprop} and \textsc{adadelta}, but we restrict our attention to \textsc{adam} to keep the presentation concise.}
which maintains moving averages of stochastic gradients and their element-wise square,
\begin{align}
\tilde{m}_t & = \beta_1 \tilde{m}_{t-1} + (1-\beta_1) g_t, & m_t &= \frac{\tilde{m}_t}{1-\beta_1^{t+1}},\\
\tilde{v}_t & = \beta_2 \tilde{v}_{t-1} + (1-\beta_2) g_t^2, & v_t &= \frac{\tilde{v}_t}{1-\beta_2^{t+1}},
\end{align}
with $\beta_1, \beta_2\in(0, 1)$ and updates
\begin{equation}
\label{eq:adam_update}
\theta_{t+1} = \theta_t - \alpha \frac{m_t}{\sqrt{v_t} + \varepsilon}
\end{equation}
with a small constant $\varepsilon>0$ preventing division by zero.
Ignoring $\varepsilon$ and assuming $\vert m_{t, i}\vert >0$ for the moment, we can rewrite the update direction as
\begin{equation}
\label{eq:rewriting_adam}
\frac{ m_t }{\sqrt{v_t}} =\frac{ \sign(m_t) }{ \sqrt{\frac{v_t}{m_t^2}} } = \sqrt{ \frac{1}{1+\frac{v_t-m_t^2}{m_t^2}} } \odot \sign(m_t),
\end{equation}
where the sign is to be understood element-wise.
Assuming that $m_t$ and $v_t$ approximate the first and second moment of the stochastic gradient---a notion that we will discuss further in \textsection\ref{estimating_variance_with_moving_averages}---$(v_t - m_t^2)$ can be seen as an estimate of the stochastic gradient variances.
The use of the \emph{non-central} second moment effectively cancels out the magnitude of $m_t$; it \emph{only} appears in the ratio $(v_t-m_t^2)/m_t^2$.
Hence, \textsc{adam} can be interpreted as a combination of two aspects:
\begin{itemize}
\item The update \emph{direction} for the $i$-th coordinate is given by the \emph{sign} of $m_{t,i}$.
\item The update \emph{magnitude} for the $i$-th coordinate is solely determined by the global step size $\alpha$ and the factor
\begin{equation}
\gamma_{t,i} := \sqrt{\frac{1}{1+\hat{\eta}_{t,i}^2}},
\end{equation}
where $\hat{\eta}_{t,i}$ is an estimate of the \emph{relative variance},
\begin{equation}
\hat\eta^2_{t,i} := \frac{v_{t, i} - m_{t,i}^2}{m_{t,i}^2} \approx\frac{\sigma^2_{t,i}}{\nabla \L_{t,i}^2} =: \eta^2_{t,i}.
\end{equation}
\end{itemize}
We will refer to the second aspect as \emph{variance adaptation}.
The variance adaptation factors shorten the update in directions of high relative variance, adapting for varying reliability of the stochastic gradient in different coordinates.

The above interpretation of \textsc{adam}'s update rule has to be viewed in contrast to existing ones.
A motivation given by \citet{Kingma2014} is that $v_t$ is a diagonal approximation to the empirical Fisher information matrix (FIM), making \textsc{adam} an approximation to natural gradient descent \citep{Amari1998}.
Apart from fundamental reservations towards the \emph{empirical} Fisher and the quality of \emph{diagonal} approximations \citep[][\textsection 11]{Martens2014}, this view is problematic because the FIM, if anything, is approximated by $v_t$, whereas \textsc{adam} adapts with the square-root $\sqrt{v_t}$.

Another possible motivation (which is not found in peer-reviewed publications but circulates the community as ``conventional wisdom'') is that \textsc{adam} performs an approximate \emph{whitening} of stochastic gradients.
However, this view hinges on the fact that \textsc{adam} divides by the square-root of the \emph{non-central} second moment, not by the standard deviation.

\subsection{Overview}

Both aspects of \textsc{adam}---taking the sign and variance adaptation---are briefly mentioned in \citet{Kingma2014}, who note that ``[t]he effective stepsize [...] is also invariant to the scale of the gradients'' and refer to $m_t/\sqrt{v_t}$ as a ``signal-to-noise ratio''.
The purpose of this work is to disentangle these two aspects in order to discuss and analyze them in isolation.

This perspective naturally suggests two alternative methods by incorporating one of the aspects while excluding the other.
Taking the sign of a stochastic gradient without any further modification gives rise to \emph{Stochastic Sign Descent} (\textsc{ssd}).
On the other hand, \emph{Stochastic Variance-Adapted Gradient} (\textsc{svag}), to be derived in \textsection\ref{va_for_grad}, applies variance adaptation directly to the stochastic gradient instead of its sign.
Together with \textsc{adam}, the momentum variants of \textsc{sgd}, \textsc{ssd}, and \textsc{svag} constitute the four possible recombinations of the sign aspect and the variance adaptation, see Fig.~\ref{fig:methods_overview}.

We proceed as follows:
Section \ref{why_sign} discusses the sign aspect.
In a simplified setting  we investigate under which circumstances the sign of a stochastic gradient is a better update direction than the stochastic gradient itself.
Section \ref{variance_adaptation} presents a principled derivation of element-wise variance adaptation factors.
Subsequently, we discuss the practical implementation of variance-adapted methods (Section \ref{practical_implementation}).
Section \ref{connection_to_generalization} draws a connection to recent work on \textsc{adam}'s effect on generalization.
Finally, Section \ref{experiments} presents experimental results.

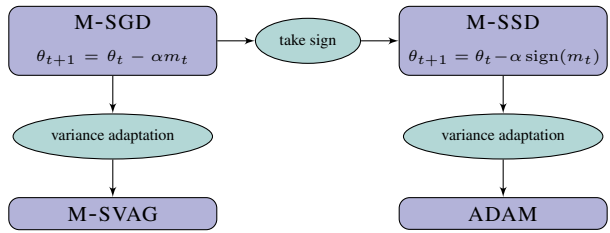
\begin{figure}
\centering
\tikzstyle{block} = [rectangle, draw, fill=blu!30,
    text width=7.2em, text centered, rounded corners]
\tikzstyle{line} = [draw, -latex']
\tikzstyle{cloud} = [draw, ellipse,fill=mpg!30]
    
\begin{tikzpicture}[align=center, node distance=.5cm]
    \node [block] (sgd) {\textsc{m-sgd}\\ {\tiny $\theta_{t+1} = \theta_t - \alpha m_t$}};
    
    \node [cloud, below=of sgd] (va1) {\tiny variance adaptation};
    \node [block, below=of va1] (svag) {\textsc{m-svag}};
    
    \node [cloud, right=of sgd] (sign) {\tiny take sign};
    \node [block, right=of sign] (ssd) {\textsc{m-ssd\\ {\tiny $\theta_{t+1} = \theta_t - \alpha \sign(m_t)$}}};
    
    \node [cloud, below=of ssd] (va2) {\tiny variance adaptation};
    \node [block, below=of va2] (adam) {\textsc{adam}};
    
    \path [line] (sgd) -- (va1);
    \path [line] (va1) -- (svag);
    
    \path [line] (sgd) -- (sign);
    \path [line] (sign) -- (ssd);
    
    \path [line] (ssd) -- (va2);
    \path [line] (va2) -- (adam);
\end{tikzpicture}
\caption{The methods under consideration in this paper. ``\textsc{m-}'' refers to the use of $m_t$ in place of $g_t$, which we colloquially refer to as the \emph{momentum variant}.
\textsc{m-svag} will be derived below.}
\label{fig:methods_overview}
\end{figure}

\subsection{Related Work}

Sign-based optimization algorithms have received some attention in the past.
\textsc{rprop} \citep{Riedmiller1993} is based on gradient signs and adapts per-element update magnitudes based on observed sign changes.
\citet{Seide2014} empirically investigate the use of stochastic gradient signs in a distributed setting with the goal of reducing communication cost.
\citet{Karimi2016} prove convergence results for sign-based methods in the \emph{non-stochastic} case.

Variance-based update directions have been proposed before, e.g., by \citet{Schaul2013}, where the variance appears together with curvature estimates in a diagonal preconditioner for \textsc{sgd}.
Their variance-dependent terms resemble the variance adaptation factors we will derive in Section \ref{variance_adaptation}.
The corresponding parts of our work complement that of \citet{Schaul2013} in various ways.
Most notably, we provide a principled motivation for variance adaptation that is independent of the update direction and use that to extend the variance adaptation to the momentum case.

A somewhat related line of research aims to obtain \emph{reduced-variance} gradient estimates \citep[e.g.,][]{Johnson2013, Defazio2014}.
This is largely orthogonal to our notion of variance adaptation, which alters the search direction to mitigate adverse effects of the (remaining) variance.

\subsection{The Sign of a Stochastic Gradient}

For later use, we briefly establish some facts about the sign\footnote{To avoid a separate zero-case, we define $\sign(0)=1$ for all theoretical considerations.
Note that $g_i\neq 0$ a.s.~if $\mathbf{var}[g_i]>0$.}
of a stochastic gradient, $s=\sign(g)$.
The distribution of the binary random variable $s_i$ is fully characterized by the \emph{success probability} $\rho_i := \mathbf{P}\left[ s_i = \sign (\nabla\L_i)\right]$, which generally depends on the distribution of $g_i$.
If we assume $g_i$ to be normally distributed, which is supported by the Central Limit Theorem applied to Eq.~\eqref{eq:stochastic_gradient}, we have
\begin{equation}
\label{eq:rho_gaussian}
\rho_i = \frac{1}{2} + \frac{1}{2} \erf\left( \frac{\vert\nabla\L_i\vert}{\sqrt{2}\sigma_i} \right),
\end{equation}
see \textsection B.1 of the supplementary material.
Note that $\rho_i$ is uniquely determined by the relative variance of $g_i$.

\section{Why the Sign?}
\label{why_sign}

Can it make sense to use the sign of a stochastic gradient as the update direction instead of the stochastic gradient itself?
This question is difficult to tackle in a general setting, but we can get an intuition using the simple, yet insightful, case of stochastic quadratic problems, where we can investigate the effects of curvature properties and noise.

\begin{problem*}[Stochastic Quadratic Problem, sQP]
\label{prop:sqp}
Consider the loss function $\ell(\theta; x)=0.5 \, (\theta-x)^T Q (\theta-x)$ with a symmetric positive definite matrix $Q\in\mathbb{R}^{d\times d}$ and ``data'' coming from the distribution $x\sim\mathcal{N}(x^\ast, \nu^2 I)$ with $\nu\in\mathbb{R}_+$.
The objective $\L(\theta)=\mathbf{E}_x[\ell(\theta; x)]$ evaluates to
\begin{equation}
\begin{split}
\L(\theta) = \frac{1}{2} (\theta-x^\ast)^TQ(\theta-x^\ast) + \frac{\nu^2}{2} \tr(Q),
\end{split}
\end{equation}
with $\nabla \L(\theta) = Q(\theta-x^\ast)$.
Stochastic gradients are given by $g(\theta)=Q(\theta-x)\sim\mathcal{N}(\nabla\L(\theta), \nu^2 QQ)$.
\end{problem*}

\subsection{Theoretical Comparison}

We compare update directions on sQPs in terms of their local expected decrease in function value from a single step.
For any stochastic direction $z$, updating from $\theta$ to $\theta + \alpha z$ results in $\mathbf{E}[\L(\theta + \alpha z)] = \L(\theta) + \alpha \nabla\L(\theta)^T \mathbf{E}[z] + \frac{\alpha^2}{2} \mathbf{E}[z^TQz]$.
For this comparison of update \emph{directions} we use the optimal step size minimizing $\mathbf{E}[\L(\theta + \alpha z)]$, which is easily found to be $\alpha_\ast=-\nabla\L(\theta)^T\mathbf{E}[z]/\mathbf{E}[z^TQz]$ and yields an expected improvement of
\begin{equation}
\label{eq:improvement_def}
\mathcal{I}(\theta) := \vert \mathbf{E}[\L(\theta + \alpha_\ast z)] - \L(\theta) \vert = \frac{(\nabla \L(\theta)^T\mathbf{E}[z])^2}{2\mathbf{E}[z^TQz]}.
\end{equation}
Locally, a larger expected improvement implies a better update direction.
We compute this quantity for \textsc{sgd} ($z=-g(\theta)$) and \textsc{ssd} ($z=-\sign(g(\theta))$) in \textsection B.2 of the supplementary material and find
\begin{align}
\label{eq:improvement_sgd_ssd}
\mathcal{I}_\textsc{sgd}(\theta) & = \frac{1}{2}\,\frac{(\nabla \L(\theta)^T\nabla\L(\theta))^2}{\nabla\L(\theta)^T Q \nabla\L(\theta) + \nu^2 \sum_{i=1}^d \lambda_i^3}, \\
\mathcal{I}_\textsc{ssd}(\theta) & \geq \frac{1}{2} \frac{\left( \sum_{i=1}^d (2\rho_i - 1) \vert \nabla\L(\theta)_i\vert \right)^2}{\sum_{i=1}^d \lambda_i} p_\text{diag}(Q),
\end{align}
where the $\lambda_i\in\mathbb{R}_+$ are the eigenvalues of $Q$ and $p_\text{diag}(Q):= (\sum_{i=1}^d \vert q_{ii}\vert) / (\sum_{i,j=1}^d \vert q_{ij}\vert)$ measures the percentage of diagonal mass of $Q$.
$\mathcal{I}_\textsc{sgd}$ and $\mathcal{I}_\textsc{ssd}$ are \emph{local} quantities, depending on $\theta$, which makes a general and conclusive comparison impossible.
However, we can draw conclusions about how properties of the sQP affect the two update directions.
We make the following two observations:

Firstly, the term $p_\text{diag}(Q)$, which features only in $\mathcal{I}_\textsc{ssd}$, relates to the orientation of the eigenbasis of $Q$.
If $Q$ is diagonal, the problem is perfectly axis-aligned and we have $p_\text{diag}(Q)=1$.
This is the obvious best case for the intrinsically axis-aligned sign update.
However, $p_\text{diag}(Q)$ can become as small as $1/d$ in the worst case and will on average (over random orientations) be $p_\text{diag}(Q) \approx 1.57/d$.
(We show these properties in \textsection B.2 of the supplementary material.)
This suggests that the sign update will have difficulties with arbitrarily-rotated eigenbases and crucially relies on the problem being ``close to axis-aligned''.

Secondly, $\mathcal{I}_\textsc{sgd}$ contains the term $\nu^2 \sum_{i=1}^d \lambda_i^3$ in which stochastic noise and the eigenspectrum of the problem \emph{interact}.
$\mathcal{I}_\textsc{ssd}$, on the other hand, has a milder dependence on the eigenvalues of $Q$ and there is no such interaction between noise and eigenspectrum.
The noise only manifests in the element-wise success probabilities $\rho_i$.

In summary, we can expect the sign direction to be beneficial for noisy, ill-conditioned problems with diagonally dominant Hessians.
It is unclear to what extent these properties hold for real problems, on which sign-based methods like \textsc{adam} are usually applied.
\citet{Becker1988} empirically investigated the first property for Hessians of simple neural network training problems and found comparably high values of $p_\text{diag}(Q)=0.1$ up to $p_\text{diag}(Q)=0.6$.
\citet{Chaudhari2016} empirically investigated the eigenspectrum in deep learning problems and found it to be very ill-conditioned with the majority of eigenvalues close to zero and a few very large ones.
However, this empirical evidence is far from conclusive.

\subsection{Experimental Evaluation}
\label{toy_problems}

\begin{figure*}
\centering
\includegraphics{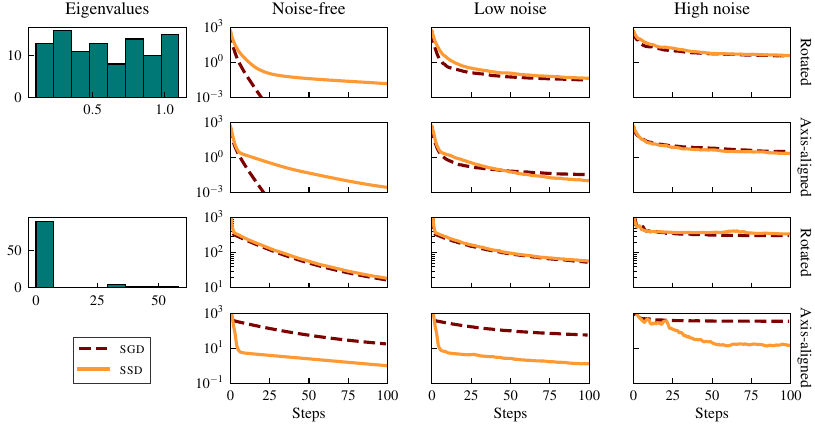}
\caption{Performance of \textsc{sgd} and \textsc{ssd} on stochastic quadratic problems.
Rows correspond to different QPs: the eigenspectrum is shown and each is used with a randomly rotated and an axis-aligned eigenbasis.
Columns correspond to different noise levels.
The individual panels show function value over number of steps.
On the well-conditioned problem, gradient descent vastly outperforms the sign-based method in the noise-free case, but the difference is evened out when noise is added.
The orientation of the eigenbasis had little effect on the comparison in the well-conditioned case.
On the ill-conditioned problem, the methods perform roughly equal when the eigenbasis is randomly rotated.
\textsc{ssd} benefits drastically from an axis-aligned eigenbasis, where it clearly outperforms \textsc{sgd}.
}
\label{fig:toy_problem_quadratic}
\end{figure*}

We verify our findings experimentally on 100-dimensional sQPs.
First, we specify a diagonal matrix $\Lambda\in\mathbb{R}^{100}$ of eigenvalues: (a) a mildly-conditioned problem with values drawn uniformly from $[0.1, 1.1]$ and (b) an ill-conditioned problem with a structured eigenspectrum simulating the one reported by \citet{Chaudhari2016} by uniformly drawing 90\% of the values from $[0,1]$ and 10\% from $[30, 60]$.
$Q$ is then defined as (a) $Q=\Lambda$ for an axis-aligned problem and (b) $Q=R\Lambda R^T$ with a random $R$ drawn uniformly among all rotation matrices \citep[see][]{Diaconis1987}.
This makes four different matrices, which we consider at noise levels $\nu\in\{0, 0.1, 4.0\}$.
We run \textsc{sgd} and \textsc{ssd} with their optimal local step sizes as previously derived.
The results, shown in Fig.~\ref{fig:toy_problem_quadratic}, confirm our theoretical findings.

\section{Variance Adaptation}
\label{variance_adaptation}

We now proceed to the second component of \textsc{adam}: variance-based element-wise step sizes.
Considering this variance adaptation in isolation from the sign aspect naturally suggests to employ it on arbitrary  update directions, for example directly on the stochastic gradient instead of its sign.
A principled motivation arises from the following consideration:

Assume we want to update in a direction $p\in\mathbb{R}^d$ (or $\sign(p)$), but only have access to an estimate $\hat{p}$ with $\mathbf{E}[\hat{p}]=p$.
We allow element-wise factors $\gamma\in\mathbb{R}^d$ and update $\gamma \odot \hat{p}$ (or $\gamma \odot \sign(\hat{p})$).
One way to make ``optimal'' use of these factors is to choose them such as to minimize the expected distance to the desired update direction.

\begin{lemma}
\label{lemma:optimal_va_factors}
Let $\hat{p}\in\mathbb{R}^d$ be a random variable with $\mathbf{E}[\hat{p}]=p$ and $\mathbf{var}[p_i]=\sigma_i^2$.
Then $\mathbf{E}[\Vert \gamma\odot\hat{p} - p\Vert_2^2]$ is minimized by
\begin{equation}
\label{eq:optimal_gamma_no_sign}
\gamma_i = \frac{\mathbf{E}[\hat{p}_i]^2}{\mathbf{E}[\hat{p}_i^2]} = \frac{p_i^2}{p_i^2 + \sigma_i^2} = \frac{1}{1+\sigma_i^2/p_i^2}
\end{equation}
and $\mathbf{E}[\Vert \gamma\odot\sign(\hat{p}) - \sign(p) \Vert_2^2]$ is minimized by
\begin{equation}
\label{eq:optimal_gamma_sign}
\gamma_i = (2\rho_i -1),
\end{equation}
where $\rho_i:=\mathbf{P}[\sign(\hat{p}_i) = \sign(p_i)]$. \hfill (Proof in \textsection B.3)
\end{lemma}

\subsection{ADAM as Variance-Adapted Sign Descent}
\label{va_for_sign}

According to Lemma \ref{lemma:optimal_va_factors}, the optimal variance adaptation factors for the sign of a stochastic gradient are $\gamma_i=2\rho_i -1$, where $\rho_i = \mathbf{P}[\sign(g_i)=\sign(\nabla \L_i)]$.
Appealing to intuition, this means that $\gamma_i$ is proportional to the success probability with a maximum of $1$ when we are certain about the sign of the gradient ($\rho_i=1$) and a minimum of $0$ in the absence of information ($\rho_i=0.5$).

Recall from Eq.~\eqref{eq:rho_gaussian} that, under the Gaussian assumption, the success probabilities are $2\rho_i -1 = \erf[(\sqrt{2} \eta_i)^{-1}]$.
Figure~\ref{fig:va_factors} shows that this term is closely approximated by $(1+\eta_i^2)^{-1/2}$, the variance adaptation terms of \textsc{adam}.
Hence, \textsc{adam} can be regarded as an approximate realization of this optimal variance adaptation scheme.
This comes with the caveat that \textsc{adam} applies these factors to $\sign(m_t)$ instead of $\sign(g_t)$.
Variance adaptation for $m_t$ will be discussed further in \textsection\ref{incorporating_momentum} and in the supplements \textsection C.2.

\begin{figure}
\centering
\includegraphics{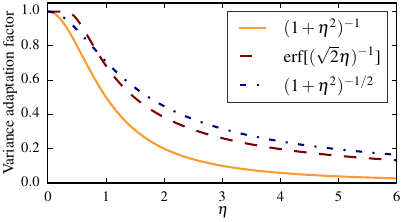}
\caption{Variance adaptation factors as functions of the relative standard deviation $\eta$.
The optimal factor for the sign of a (Gaussian) stochastic gradient is $\erf[(\sqrt{2}\eta)^{-1}]$, which is closely approximated by $(1+\eta^2)^{-1/2}$, the factor implicitly employed by \textsc{adam}. $(1+\eta^2)^{-1}$ is the optimal factor for a stochastic gradient.}
\label{fig:va_factors}
\end{figure}

\subsection{Stochastic Variance-Adapted Gradient (SVAG)}
\label{va_for_grad}

Applying Eq.~\eqref{eq:optimal_gamma_no_sign} to $\hat{p}=g$, the optimal variance adaptation factors for a stochastic gradient are found to be
\begin{equation}
\label{eq:optimal_gamma_sgd}
\gamma^g_i = \frac{\nabla \L_i^2}{\nabla \L_i^2 + \sigma_i^2} =  \frac{1}{1+ \sigma_i^2/\nabla \L_i^2} =  \frac{1}{1+ \eta_i^2}.
\end{equation}
A term of this form also appears, together with diagonal curvature estimates, in \citet{Schaul2013}.
We refer to the method updating along $\gamma^g \odot g$ as \emph{Stochastic Variance-Adapted Gradient} (\textsc{svag}).
To support intuition, Fig.~\ref{fig:conceptual_sketch} shows a conceptual sketch of this variance adaptation scheme.

Variance adaptation of this form guarantees convergence \emph{without} manually decreasing the global step size.
We recover the $\mathcal{O}(1/t)$ rate of \textsc{sgd} for smooth, strongly convex functions.
We emphasize that this result considers an \emph{idealized} version of \textsc{svag} with exact $\gamma^g_i$.
It should be considered as a motivation for this variance adaptation strategy, not a statement about its performance with estimated variance adaptation factors.
\begin{theorem}
\label{theorem:convergence_svag}
Let $f\colon\mathbb{R}^d\to\mathbb{R}$ be $\mu$-strongly convex and $L$-smooth.
We update $\theta_{t+1} = \theta_t - \alpha (\gamma_t \odot g_t)$, with stochastic gradients $\mathbf{E}[g_t\vert\theta_t]=\nabla f_t$, $\mathbf{var}[g_{t,i}\vert \theta_t]=\sigma_{t,i}^2$, variance adaptation factors $\gamma_{t,i}=\nabla f_{t,i}^2/(\nabla f_{t,i}^2 + \sigma_{t,i}^2)$, and a global step size $\alpha=1/L$. Assume that there are constants $c_v, M_v>0$ such that $\sum_{i=1}^d \sigma_{t,i}^2 \leq c_v \Vert \nabla f_t\Vert^2 + M_v$. Then
\begin{equation}
\mathbf{E}[f(\theta_t)-f_\ast] \in \mathcal{O}\left( \frac{1}{t} \right),
\end{equation}
where $f_\ast$ is the minimum value of $f$. \hfill(Proof in \textsection B.4)
\end{theorem}

The assumption $\sum_{i=1}^d \sigma_{t,i}^2 \leq c_v \Vert \nabla f_t\Vert^2 + M_v$ is a mild restriction on the variances, allowing them to be non-zero everywhere and to grow quadratically in the gradient norm.

\begin{figure}
\centering
\includegraphics[scale=1.0]{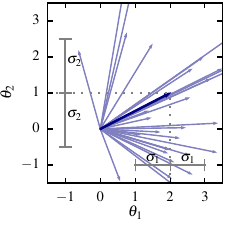}
\hspace{8pt}
\includegraphics[scale=1.0]{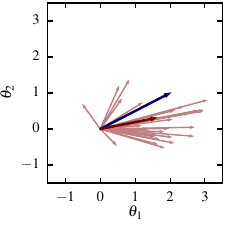}
\caption{Conceptual sketch of variance-adapted stochastic gradients.
The left panel shows the true gradient $\nabla\L=(2,1)$ and stochastic gradients scattered around it with $(\sigma_1, \sigma_2)=(1, 1.5)$.
In the right panel, we scale the $i$-th coordinate by $(1+\eta_i^2)^{-1}$.
In this example, the $\theta_2$-coordinate has much higher relative variance ($\eta_2^2 = 2.25$) than the $\theta_1$-coordinate ($\eta^2_1 = 0.25$) and is thus shortened.
This reduces the variance of the update direction at the expense of biasing it away from the true gradient in expectation.
}
\label{fig:conceptual_sketch}
\end{figure}

\section{Practical Implementation of M-SVAG}

\label{practical_implementation}

Section \ref{variance_adaptation} has introduced the general idea of variance adaptation; we now discuss its practical implementation.
For the sake of a concise presentation, we focus on one particular variance-adapted method, \textsc{m-svag}, which applies variance adaptation to the update direction $m_t$.
This method is of particular interest due to its relationship to \textsc{adam} outlined in Figure \ref{fig:methods_overview}.
Many of the following considerations correspondingly apply to other variance-adapted methods, e.g., \textsc{svag} and variants of \textsc{adam}, some of which are discussed and evaluated in the supplementary material (\textsection C).

\subsection{Estimating Gradient Variance}
\label{estimating_variance_with_moving_averages}

In practice, the optimal variance adaptation factors are unknown and have to be estimated.
A key ingredient is an estimate of the stochastic gradient variance.
We have argued in the introduction that \textsc{adam} obtains such an estimate from moving averages, $\sigma_{t,i}^2 \approx v_{t,i} - m_{t,i}^2$.
The underlying assumption is that the distribution of stochastic gradients is approximately constant over the effective time horizon of the exponential moving average, making $m_t$ and $v_t$ estimates of the first and second moment of $g_t$, respectively:
\begin{assumption}
\label{assumption:iid_grads}
At step $t$, assume 
\begin{equation}
\mathbf{E}[m_{t,i}] \approx \nabla\L_{t,i}, \quad \mathbf{E}[v_{t,i}] \approx  \nabla\L_{t,i}^2 + \sigma^2_{t,i}.
\end{equation}
\end{assumption}
While this can only ever hold approximately, Assumption \ref{assumption:iid_grads} is the tool we need to obtain gradient variance estimates from past gradient observations.
It will be more realistic in the case of high noise and small step size, where the variation between successive stochastic gradients is dominated by stochasticity rather than change in the true gradient. 

We make two modifications to \textsc{adam}'s variance estimate.
First, we will use the same moving average constant $\beta_1=\beta_2=\beta$ for $m_t$ and $v_t$.
This constant should define the effective range for which we implicitly assume the stochastic gradients to come from the same distribution, making different constants for the first and second moment implausible.

Secondly, we adapt for a systematic bias in the variance estimate.
As we show in \textsection B.5, under Assumption 1,
\begin{gather}
\label{eq:expectation_of_m_squared}
\mathbf{E}[m_{t,i}^2] \approx \nabla \L_{t,i}^2 + \rho(\beta, t) \sigma_{t,i}^2,\\
\label{eq:definition_rho}
\rho(\beta, t) := \frac{(1-\beta)(1+\beta^{t+1})}{(1+\beta)(1-\beta^{t+1})},
\end{gather}
and consequently $\mathbf{E}[v_{t,i}-m_{t,i}^2] \approx (1-\rho(\beta, t)) \, \sigma_{t,i}^2$.
We correct for this bias and use the variance estimate
\begin{equation}
\label{eq:bias_corrected_ema_estimate}
\hat{s}_t := \frac{1}{1-\rho(\beta, t)} (v_t - m_t^2).
\end{equation}

\emph{Mini-Batch Gradient Variance Estimates:}
An alternative variance estimate can be computed locally ``within'' a single mini-batch, see \textsection D of the supplements.
We have experimented with both estimators and found the resulting methods to have similar performance.
For the main paper, we stick to the moving average variant for its ease of implementation and direct correspondence with \textsc{adam}.
We present experiments with the mini-batch variant in the supplementary material.
These demonstrate the merit of variance adaptation \emph{irrespective} of how the variance is estimated.

\subsection{Estimating the Variance Adaptation Factors}
\label{estimating_svag_factors}

The gradient variance itself is not of primary interest; we have to estimate the variance adaptation factors, given by Eq.~\eqref{eq:optimal_gamma_sgd} in the case of \textsc{svag}.
We propose to use the estimate
\begin{equation}
\label{eq:estimated_va_factors_grad}
\hat{\gamma}^g_t = \frac{1}{1	+ \hat{s}_t/m_t^2} = \frac{m_t^2}{m_t^2 + \hat{s}_t}.
\end{equation}
While $\hat{\gamma}^g_t$ is an intuitive quantity, it is \emph{not} an unbiased estimate of the exact variance adaptation factors as defined in Eq.~\eqref{eq:optimal_gamma_sgd}.
To our knowledge, unbiased estimation of the exact factors is intractable.
We have experimented with several partial bias correction terms but found them to have destabilizing effects.


\subsection{Incorporating Momentum}
\label{incorporating_momentum}

So far, we have considered variance adaptation for the update direction $g_t$.
In practice, we may want to update in the direction of $m_t$ to incorporate momentum.\footnote{
Our use of the term \emph{momentum} is somewhat colloquial.
To highlight the relationship with \textsc{adam} (Fig.~\ref{fig:methods_overview}), we have defined \textsc{m-sgd} as the method using the update direction $m_t$, which is a rescaled version of \textsc{sgd} with momentum.
\textsc{m-svag} applies variance adaptation to $m_t$.
This is not to be confused with the application of momentum acceleration \citep{Polyak1964, Nesterov1983} \emph{on top} of a \textsc{svag} update.
}
According to Lemma \ref{lemma:optimal_va_factors}, the variance adaptation factors should then be determined by the relative of variance of $m_t$.

Once more adopting Assumption \ref{assumption:iid_grads}, we  have $\mathbf{E}[m_t]\approx \nabla \L_t$ and $\mathbf{var}[m_{t,i}] \approx \rho(\beta, t) \sigma_{t,i}^2$, the latter being due to Eq.~\eqref{eq:expectation_of_m_squared}.
Hence, the relative variance of $m_t$ is $\rho(\beta, t)$ times that of $g_t$, such that the optimal variance adaptation factors for the update direction $m_t$ according to Lemma \ref{lemma:optimal_va_factors} are
\begin{equation}
\gamma_{t,i}^m = \frac{1}{1+\rho(\beta, t) \sigma_{t,i}^2/\nabla \L_{t,i}^2}.
\end{equation}
We use the following estimate thereof:
\begin{equation}
\label{eq:estimated_va_factors_msvag}
\hat{\gamma}^m_t = \frac{1}{1	+  \rho(\beta, t) \, \hat{s}_t/m_t^2 } = \frac{m_t^2}{m_t^2 + \rho(\beta, t) \, \hat{s}_t}.
\end{equation}
Note that $m_t$ now serves a double purpose:
It determines the base update direction and, at the same time, is used to obtain an estimate of the gradient variance.


\subsection{Details}
\label{implementation_details}

Note that Eq.~\eqref{eq:bias_corrected_ema_estimate} is ill-defined for $t=0$, since $\rho(\beta,0)=0$.
We use $\hat{s}_0=0$ for the first iteration, making the initial step of \textsc{m-svag} coincide with an \textsc{sgd}-step.
One final detail concerns a possible division by zero in Eq.~\eqref{eq:estimated_va_factors_msvag}.
Unlike \textsc{adam}, we do not add a constant offset $\varepsilon$ in the denominator.
A division by zero only occurs when $m_{t,i}=v_{t,i}=0$;
we check for this case and perform no update, since $m_{t,i}=0$.

This completes the description of our implementation of \textsc{m-svag}.
Alg.~\ref{alg:msvag_ema} provides pseudo-code (ignoring the details discussed in \textsection\ref{implementation_details} for readability).

\begin{algorithm}
\footnotesize
\caption{\textsc{m-svag}}
\label{alg:msvag_ema}
\begin{algorithmic}
\STATE {\bfseries Input:} $\theta_0\in\mathbb{R}^d$, $\alpha>0$, $\beta \in [0,1]$, $T\in\mathbb{N}$
\STATE Initialize $\theta\gets \theta_0$, $\tilde{m}\gets 0$, $\tilde{v}\gets 0$
\FOR{$t=0,\dotsc, T-1$}
  \STATE $\tilde{m}\gets \beta \tilde{m} + (1-\beta) g(\theta)$, \quad $\tilde{v}\gets \beta \tilde{v} + (1-\beta)g(\theta)^2$
  \vspace{1pt}
  \STATE $m \gets (1-\beta^{t+1})^{-1} \tilde{m}$, \quad $v \gets (1-\beta^{t+1})^{-1}\tilde{v}$
  \vspace{1pt}
  \STATE $s\gets (1-\rho(\beta, t))^{-1} (v-m^2)$ 
  \vspace{1pt}
  \STATE $\gamma \gets m^2/(m^2 + \rho(\beta, t) s)$
  \vspace{1pt}
  \STATE $\theta \gets \theta - \alpha (\gamma\odot m)$
\ENDFOR
\end{algorithmic}
\emph{Note:} \textsc{m-svag} exposes two hyperparameters, $\alpha$ and $\beta$.
\end{algorithm}

\section{Connection to Generalization}

\label{connection_to_generalization}


Of late, the question of the effect of the optimization algorithm on \emph{generalization} has received increased attention. 
Especially in deep learning, different optimizers might find solutions with varying generalization performance.
Recently, \mbox{\citet{Wilson2017}} have argued that ``adaptive methods'' (referring to \textsc{adagrad}, \textsc{rmsprop}, and \textsc{adam}) have adverse effects on generalization compared to ``non-adaptive methods'' (gradient descent, \textsc{sgd}, and their momentum variants).
In addition to an extensive empirical validation of that claim, the authors make a theoretical argument using a binary least-squares classification problem,
\begin{equation}
\label{eq:least_squares_classification}
R(\theta) = \frac{1}{n} \sum_{i=1}^n \frac{1}{2} (x_i^T \theta - y_i)^2 = \frac{1}{2n} \Vert X\theta - y \Vert^2,
\end{equation}
with $n$ data points $(x_i, y_i) \in  \mathbb{R}^d\times \{\pm 1\}$, stacked in a matrix $X\in\mathbb{R}^{n\times d}$ and a label vector $y\in\{\pm 1\}^n$.
For this problem class, the non-adaptive methods provably converge to the max-margin solution, which we expect to have favorable generalization properties.
In contrast to that, \citet{Wilson2017} show that---for \emph{some} instances of this problem class--- the adaptive methods converge to solutions that generalize arbitrarily bad to unseen data.
The authors construct such problematic instances using the following Lemma.
\begin{lemma}[Lemma 3.1 in \citet{Wilson2017}]
\label{lemma:wilson_lemma}
Suppose $[X^T y]_i\neq 0$ for $i=1,\dotsc, d$, and there exists $c\in\mathbb{R}$ such that $X \sign(X^T y) = cy$. Then, when initialized at $\theta_0 = 0$, the iterates generated by full-batch \textsc{adagrad}, \textsc{adam}, and \textsc{rmsprop} on the objective \eqref{eq:least_squares_classification} satisfy $\theta_t \propto \sign(X^T y)$.
\end{lemma}
Intriguingly, as we show in \textsection B.6 of the supplementary material, this statement easily extends to sign descent, i.e., the method updating $\theta_{t+1} = \theta_t - \alpha \sign(\nabla R(\theta_t))$.
\begin{lemma}
\label{lemma:wilson_lemma_extended_to_sd}
Under the assumptions of Lemma \ref{lemma:wilson_lemma}, the iterates generated by sign descent satisfy $\theta_t \propto \sign(X^T y)$.
\end{lemma}
On the other hand, this does \emph{not} extend to \textsc{m-svag}, an \emph{adaptive} method by any standard.
As noted before, the first step of \textsc{m-svag} coincides with a gradient descent step.
The iterates generated by \textsc{m-svag} will, thus, not generally be proportional to $\sign(X^Ty)$. 
While this does by no means imply that it converges to the max-margin solution or has otherwise favorable generalization properties, the construction of \citet{Wilson2017} does \emph{not} apply to \textsc{m-svag}.

This suggests that it is the sign that impedes generalization in the examples constructed by \citet{Wilson2017}, rather than the element-wise adaptivity as such.
Our experiments substantiate this suspicion.
The fact that all currently popular adaptive methods are also sign-based has led to a conflation of these two aspects.
The main motivation for this work was to disentangle them.

\section{Experiments}
\label{experiments}

We experimentally compare \textsc{m-svag} and \textsc{adam} to their non-variance-adapted counterparts \textsc{m-sgd} and \textsc{m-ssd} (Alg.~\ref{alg:msgd_and_mssd}).
Since these are the four possible recombinations of the sign and the variance adaptation (Fig.~\ref{fig:methods_overview}), this comparison allows us to separate the effects of the two aspects.

\begin{algorithm}[t]
\footnotesize
\caption{\colorbox{blu!20}{\textsc{m-sgd}} and \colorbox{mpg!50}{\textsc{m-ssd}}}
\label{alg:msgd_and_mssd}
\begin{algorithmic}
\STATE {\bfseries Input:} $\theta_0\in\mathbb{R}^d$, $\alpha>0$, $\beta\in [0,1]$, $T\in\mathbb{N}$
\STATE Initialize $\theta \gets \theta_0$, $\tilde{m}\gets 0$
\FOR{$t=0,\dotsc, T-1$}
  \STATE $\tilde{m}\gets \beta \tilde{m} + (1-\beta) g(\theta)$
  \STATE \colorbox{blu!20}{$m\gets (1-\beta^{t+1})^{-1} \tilde{m}$}
  \STATE \colorbox{blu!20}{$\theta \gets \theta - \alpha m$} \quad \colorbox{mpg!50}{$\theta \gets \theta - \alpha \sign(\tilde{m})$}
\ENDFOR
\end{algorithmic}
\end{algorithm}

\subsection{Experimental Set-Up}

We evaluated the four methods on the following problems:
\begin{itemize}
\item[P1] A vanilla convolutional neural network (CNN) with two convolutional and two fully-connected layers on the Fashion-\textsc{mnist} data set \citep{Xiao2017}.
\item[P2] A vanilla CNN with three convolutional and three fully-connected layers on \textsc{cifar-10} \citep{Krizhevsky2009}.
\item[P3] The wide residual network WRN-40-4 architecture of \citet{Zagoruyko2016} on \textsc{cifar-100}.
\item[P4] A two-layer LSTM \citep{Hochreiter1997} for character-level language modelling on Tolstoy's \emph{War and Peace}.
\end{itemize}
A detailed description of all network architectures has been moved to \textsection A of the supplementary material.

For all experiments, we used $\beta=0.9$ for \textsc{m-sgd}, \textsc{m-ssd} and \textsc{m-svag} and default parameters ($\beta_1=0.9, \beta_2=0.999, \varepsilon=10^{-8}$) for \textsc{adam}.
The global step size $\alpha$ was tuned for each method individually by first finding the maximal stable step size by trial and error, then searching downwards.
We selected the one that yielded maximal test accuracy within a fixed number of training steps; a scenario close to an actual application of the methods by a practitioner.
(Loss and accuracy have been evaluated at a fixed interval on the full test set as well as on an equally-sized portion of the training set).
Experiments with the best step size have been replicated ten times with different random seeds.
While (P1) and (P2) were trained with constant $\alpha$, we used a decrease schedule for (P3) and (P4), which was fixed in advance for all methods.
Full details can be found in \textsection A of the supplements.

\begin{figure*}[t]
\centering
\includegraphics{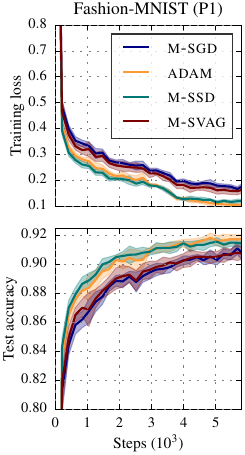}
\includegraphics{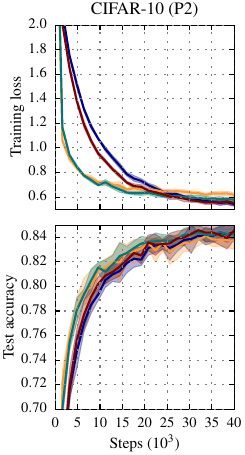}
\includegraphics{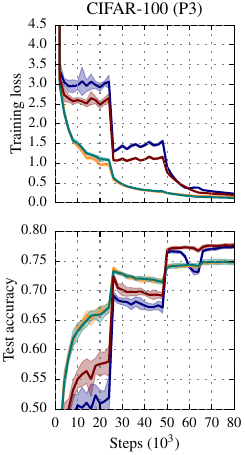}
\includegraphics{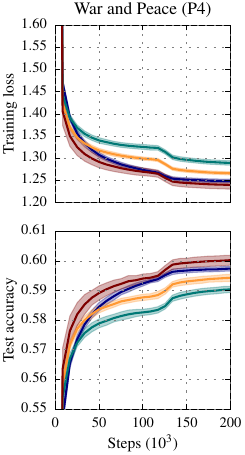}
\caption{Experimental results on the four test problems.
Plots display training loss and test accuracy over the number of steps.
Curves for the different optimization methods are color-coded.
The shaded area spans one standard deviation, obtained from ten replications with different random seeds.}
\label{fig:results}
\end{figure*}


\subsection{Results}

Fig.~\ref{fig:results} shows results.
We make four main observations.

\paragraph{1) The sign aspect dominates}
With the exception of (P4), the performance of the four methods distinctly clusters into sign-based and non-sign-based methods.
Of the two components of \textsc{adam} identified in \textsection\ref{new_perspective_on_adam}, the sign aspect seems to be by far the dominant one, accounting for most of the difference between \textsc{adam} and \textsc{m-sgd}.
\textsc{adam} and \textsc{m-ssd} display surprisingly similar performance; an observation that might inform practitioners' choice of algorithm, especially for very high-dimensional problems, where \textsc{adam}'s additional memory requirements are an issue.

\paragraph{2) The usefulness of the sign is problem-dependent}
Considering only training loss, the two sign-based methods clearly outperform the two non-sign-based methods on problems (P1) and (P3).
On (P2), \textsc{adam} and \textsc{m-ssd} make rapid initial progress, but later plateau and are undercut by \textsc{m-sgd} and \textsc{m-svag}.
On the language modelling task (P4) the non-sign-based methods show superior performance.
Relating to our analysis in Section \ref{why_sign}, this shows that the usefulness of sign-based methods depends on the particular problem at hand.

\paragraph{3) Variance adaptation helps}
In all experiments, the variance-adapted variants perform at least as good as, and often better than, their ``base algorithms''.
The magnitude of the effect varies. For example, \textsc{adam} and \textsc{m-ssd} have identical performance on (P3), but \textsc{m-svag} significantly outperforms \textsc{m-sgd} on (P3) as well as (P4).

\paragraph{4) Generalization effects are caused by the sign}
The \textsc{cifar-100} example (P3) displays similar effects as reported by \citet{Wilson2017}: \textsc{adam} vastly outperforms \textsc{m-sgd} in training loss, but has significantly worse test performance.
Observe that \textsc{m-ssd} behaves almost identical to \textsc{adam} in both train and test and, thus, displays the same generalization-harming effects.
\textsc{m-svag}, on the other hand, improves upon \textsc{m-sgd} and, in particular, does not display any adverse effects on generalization.
This corroborates the suspicion raised in \textsection\ref{connection_to_generalization} that the generalization-harming effects of \textsc{adam} are caused by the sign aspect rather than the element-wise adaptive step sizes.

\section{Conclusion}

We have argued that \textsc{adam} combines two components: taking signs and variance adaptation.
Our experiments show that the sign aspect is by far the dominant one, but its usefulness is problem-dependent.
Our theoretical analysis suggests that it depends on the interplay of stochasticity, the conditioning of the problem, and its axis-alignment.
Sign-based methods also seem to have an adverse effect on the generalization performance of the obtained solution; a possible starting point for further research into the generalization effects of optimization algorithms.

The second aspect, variance adaptation, is not restricted to \textsc{adam} but can be applied to any update direction.
We have provided a general motivation for variance adaptation factors that is independent of the update direction.
In particular, we introduced \textsc{m-svag}, a variance-adapted variant of momentum \textsc{sgd}, which is a useful addition to the practitioner's toolbox for problems where sign-based methods like \textsc{adam} fail.
A TensorFlow \citep{Tensorflow2015} implementation can be found at \url{https://github.com/lballes/msvag}.

\section*{Acknowledgements}
The authors thank Maren Mahsereci for helpful discussions.
The authors acknowledge financial support by the European Research Council through ERC StG Action 757275 / PANAMA during a part of the project.
Lukas Balles kindly acknowledges the support of the International Max Planck Research School for Intelligent Systems (IMPRS-IS).

%


\bibliography{references}
\bibliographystyle{icml2018}

\clearpage
\appendix
\section*{\centering ---Supplementary Material---}
\input{supplements_content.tex}

\end{document}

%% file: supplements_content.tex
\section{Experiments}
\label{appendix_experiments}

\subsection{Network Architectures}

\paragraph{Fashion-MNIST}
We trained a simple convolutional neural network with two convolutional layers (size 5$\times$5, 32 and 64 filters, respectively), each followed by max-pooling over 3$\times$3 areas with stride 2, and a fully-connected layer with 1024 units.
ReLU activation was used for all layers.
The output layer has 10 units with softmax activation.
We used cross-entropy loss, without any additional regularization, and a mini-batch size of 64.
We trained for a total of 6000 steps with a constant global step size $\alpha$.

\paragraph{CIFAR-10}
We trained a CNN with three convolutional layers (64 filters of size 5$\times$5, 96 filters of size 3$\times$3, and 128 filters of size 3$\times$3) interspersed with max-pooling over 3$\times$3 areas with stride 2 and followed by two fully-connected layers with 512 and 256 units.
ReLU activation was used for all layers.
The output layer has 10 units with softmax activation.
We used cross-entropy loss function and applied $L_2$-regularization on all weights, but not the biases.
During training we performed some standard data augmentation operations (random cropping of sub-images, left-right mirroring, color distortion) on the input images.
We used a batch size of 128 and trained for a total of 40k steps with a constant global step size $\alpha$.

\paragraph{CIFAR-100}
We use the WRN-40-4 architecture of \citet{Zagoruyko2016}; details can be found in the original paper.
We used cross-entropy loss and applied $L_2$-regularization on all weights, but not the biases.
We used the same data augmentation operations as for \textsc{cifar-10}, a batch size of 128, and trained for 80k steps.
For the global step size $\alpha$, we used the decrease schedule suggested by \citet{Zagoruyko2016}, which amounts to multiplying with a factor of 0.2  after 24k, 48k, and 64k steps.
TensorFlow code was adapted from \url{https://github.com/dalgu90/wrn-tensorflow}.

\paragraph{War and Peace}
We preprocessed \emph{War and Peace}, extracting a vocabulary of 83 characters.
The language model is a two-layer LSTM with 128 hidden units each.
We used a sequence length of 50 characters and a batch size of 50.
Drop-out regularization was applied during training.
We trained for 200k steps; the global step size $\alpha$ was multiplied with a factor of 0.1 after 125k steps.
TensorFlow code was adapted from \url{https://github.com/sherjilozair/char-rnn-tensorflow}.

\subsection{Step Size Tuning}

Step sizes $\alpha$ (initial step sizes for the experiments with a step size decrease schedule) for each optimizer have been tuned by first finding the maximal stable step size by trial and error and then searching downwards over multiple orders of magnitude, testing $6\cdot 10^m$, $3\cdot 10^m$, and $1\cdot 10^m$ for order of magnitude $m$.
We evaluated loss and accuracy on the full test set (as well as on an equally-sized portion of the training set) at a constant interval and selected the best-performing step size for each method in terms of maximally reached test accuracy.
Using the best choice, we replicated the experiment ten times with different random seeds, randomizing the parameter initialization, data set shuffling, drop-out, et cetera.
In some rare cases where the accuracies for two different step sizes were very close, we replicated both and then chose the one with the higher maximum mean accuracy.

The following list shows all explored step sizes, with the ``winner'' in bold face.

\textbf{Problem 1: Fashion-\textsc{mnist}}\\
\textsc{m-sgd}:\\
 $3, 1, 6\cdot 10^{-1}, 3\cdot 10^{-1}, \mathbf{1\cdot 10^{-1}}, 6\cdot 10^{-2}, 3\cdot 10^{-2}, 1\cdot 10^{-2}, 6\cdot 10^{-3}, 3\cdot 10^{-3}$\\
\textsc{adam}:\\
$3\cdot 10^{-2}, 10^{-2}, 6\cdot 10^{-3}, 3\cdot 10^{-3}, \mathbf{1\cdot 10^{-3}}, 6\cdot 10^{-4}, 3\cdot 10^{-4}, 1\cdot 10^{-4}$\\
\textsc{m-ssd}:\\
$10^{-2}, 6\cdot 10^{-3}, 3\cdot 10^{-3}, 1\cdot 10^{-3}, 6\cdot 10^{-4}, \mathbf{3\cdot 10^{-4}}, 1\cdot 10^{-4}$\\
\textsc{m-svag}:\\
$3, 1, 6\cdot 10^{-1}, \mathbf{3\cdot 10^{-1}}, 1\cdot 10^{-1}, 6\cdot 10^{-2}, 3\cdot 10^{-2}, 1\cdot 10^{-2}, 6\cdot 10^{-3}, 3\cdot 10^{-3}$

\textbf{Problem 2: \textsc{cifar-10}}\\
\textsc{m-sgd}:\\ $6\cdot 10^{-1}, 3\cdot 10^{-1}, 1\cdot 10^{-1}, 6\cdot 10^{-2}, \mathbf{3\cdot 10^{-2}}, 1\cdot 10^{-2}, 6\cdot 10^{-3}, 3\cdot 10^{-3}$\\
\textsc{adam}:\\
$6\cdot 10^{-3}, 3\cdot 10^{-3}, 1\cdot 10^{-3}, \mathbf{6\cdot 10^{-4}}, 3\cdot 10^{-4}, 1\cdot 10^{-4}, 6\cdot 10^{-5}$\\
\textsc{m-ssd}:\\
$6\cdot 10^{-3}, 3\cdot 10^{-3}, 1\cdot 10^{-3}, 6\cdot 10^{-4}, 3\cdot 10^{-4}, \mathbf{1\cdot 10^{-4}}, 6\cdot 10^{-5}, 3\cdot 10^{-5}$\\
\textsc{m-svag}:\\
$1, 6\cdot 10^{-1}, 3\cdot 10^{-1}, 1\cdot 10^{-1}, \mathbf{6\cdot 10^{-2}}, 3\cdot 10^{-2}, 1\cdot 10^{-2}, 6\cdot 10^{-3}$

\textbf{Problem 3: \textsc{cifar-100}}\\
\textsc{m-sgd}:\\
 $6, \mathbf{3}, 1, 6\cdot 10^{-1}, 3\cdot 10^{-1}, 1\cdot 10^{-1}, 6\cdot 10^{-2}, \mathbf{3\cdot 10^{-2}}, 1\cdot 10^{-2}$\\
\textsc{adam}:\\
$1\cdot 10^{-2}, 6\cdot 10^{-3}, 3\cdot 10^{-3}, 1\cdot 10^{-3}, 6\cdot 10^{-4}, \mathbf{3\cdot 10^{-4}}, 1\cdot 10^{-4}, 6\cdot 10^{-5}, 3\cdot10^{-5}$\\
\textsc{m-ssd}:\\
$1\cdot 10^{-2}, 6\cdot 10^{-3}, 3\cdot 10^{-3}, 1\cdot 10^{-3}, 6\cdot 10^{-4}, 3\cdot 10^{-4}, \mathbf{1\cdot 10^{-4}}, 6\cdot 10^{-5}, 3\cdot10^{-5}$\\
\textsc{m-svag}:\\
$6, \mathbf{3}, 1, 6\cdot 10^{-1}, 3\cdot 10^{-1}, 1\cdot 10^{-1}, 6\cdot 10^{-2}, \mathbf{3\cdot 10^{-2}}, 1\cdot 10^{-2}$

\textbf{Problem 4: War and Peace}\\
\textsc{m-sgd}:\\
$10, 6, \mathbf{3}, 1, 6\cdot 10^{-1}, 3\cdot 10^{-1}, 1\cdot 10^{-1}, 6\cdot 10^{-2}$\\
\textsc{adam}:\\
$1\cdot 10^{-2}, 6\cdot 10^{-3}, \mathbf{3\cdot 10^{-3}}, 1\cdot 10^{-3}, 6\cdot 10^{-4}, 3\cdot 10^{-4}, 1\cdot 10^{-4}, 6\cdot 10^{-5}$\\
\textsc{m-ssd}:\\
$1\cdot 10^{-2}, 6\cdot 10^{-3}, 3\cdot 10^{-3}, \mathbf{1\cdot 10^{-3}}, 6\cdot 10^{-4}, 3\cdot 10^{-4}, 1\cdot 10^{-4}, 6\cdot 10^{-5}$\\
\textsc{m-svag}:\\
$30,\mathbf{10}, 6, 3, 1, 6\cdot 10^{-1}, 3\cdot 10^{-1}, 1\cdot 10^{-1}$

\section{Mathematical Details}

\subsection{The Sign of a Stochastic Gradient}
\label{sign_of_g}

We have stated in the main text that the sign of a stochastic gradient, $s(\theta)=\sign(g(\theta))$, has success probabilities
\begin{equation}
\begin{split}
\rho_i & :=\mathbf{P}[s(\theta)_i = \sign(\nabla\L(\theta)_i)] \\
& = \frac{1}{2} + \frac{1}{2} \erf\left( \frac{\vert \nabla\L(\theta)_i \vert}{\sqrt{2} \sigma(\theta)_i} \right)
\end{split}
\end{equation}
under the assumption that $g\sim\mathcal{N}(\nabla\L, \Sigma)$.
The following Lemma formally proves this statement and Figure \ref{fig:gaussian_overlap} provides a pictorial illustration.
\begin{lemma}
\label{lemma_rho_gaussian}
If $X\sim\mathcal{N}(\mu,\sigma^2)$ then
\begin{equation}
\mathbf{P}[\sign(X) = \sign(\mu)] = \frac{1}{2} \left( 1+ \erf\left( \frac{\vert \mu \vert}{\sqrt{2}\sigma} \right) \right).
\end{equation}
\end{lemma}
\begin{proof}
Define $\rho:=\mathbf{P}[\sign(X) = \sign(\mu)]$. The cumulative density function (cdf) of $X\sim\mathcal{N}(\mu, \sigma^2)$ is $\mathbf{P}[X\leq x] = \Phi((x-\mu)/\sigma)$, where $\Phi(z) = 0.5  (1+\erf(z/\sqrt{2}))$ is the cdf of the standard normal distribution. If $\mu <0$, then
\begin{equation}
\begin{split}
\rho & = \mathbf{P}[X<0] = \Phi\left(\frac{0- \mu}{\sigma} \right) \\
& = \frac{1}{2} \left( 1+\erf\left(\frac{-\mu}{\sqrt{2} \sigma} \right) \right).
\end{split}
\end{equation}
If $\mu > 0$, then 
\begin{equation}
\begin{split}
\rho & = \mathbf{P}[X>0] = 1 - \mathbf{P}[X \leq 0] = 1 - \Phi\left(\frac{0- \mu}{\sigma} \right) \\
& = 1 - \frac{1}{2} \left( 1+\erf\left(\frac{-\mu}{\sqrt{2} \sigma} \right) \right) \\
& = \frac{1}{2} \left( 1+\erf\left(\frac{\mu}{\sqrt{2} \sigma} \right) \right),
\end{split}
\end{equation}
where the last step used the anti-symmetry of the error function.
\end{proof}
\begin{figure*}
\centering
\includegraphics{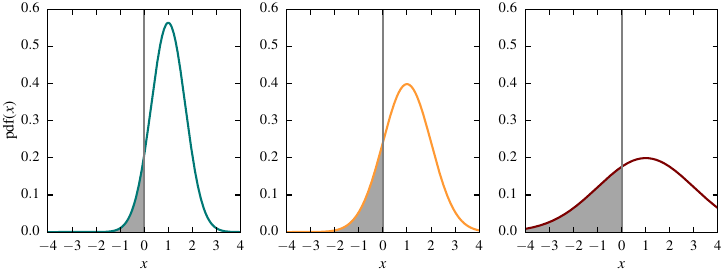}
\caption{Probability density functions (pdf) of three Gaussian distributions, all with $\mu=1$, but different variances $\sigma^2=0.5$ (left), $\sigma^2=1.0$ (middle), $\sigma^2=4.0$ (right). The shaded area under the curve corresponds to the probability that a sample from the distribution has the opposite sign than its mean. For the Gaussian distribution, this probability is uniquely determined by the fraction $\sigma / \vert\mu\vert$, as shown in Lemma \ref{lemma_rho_gaussian}.}
\label{fig:gaussian_overlap}
\end{figure*}

\subsection{Analysis on Stochastic QPs}
\label{why_sign_supplements}

\subsubsection{Derivation of $\mathcal{I}_\textsc{sgd}$ and $\mathcal{I}_\textsc{ssd}$}
We derive the expressions in Eq.~\eqref{eq:improvement_sgd_ssd}, dropping the fixed $\theta$ from the notation for readability.

For \textsc{sgd}, we have $\mathbf{E}[g]=\nabla \L$ and $\mathbf{E}[g^TQg]=\nabla \L^TQ\nabla \L + \tr(Q\mathbf{cov}[g])$, which is a general fact for quadratic forms of random variables.
For the stochastic QP the gradient covariance is $\mathbf{cov}[g] = \nu^2 QQ$, thus $\tr(Q\mathbf{cov}[g])=\nu^2 \tr(QQQ)=\nu^2 \sum_i \lambda_i^3$. Plugging everything into Eq.~\eqref{eq:improvement_def} yields
\begin{equation}
\mathcal{I}_\textsc{sgd} = \frac{(\nabla \L^T\nabla\L)^2}{\nabla\L^T Q \nabla\L + \nu^2 \sum_{i=1}^d \lambda_i^3}.
\end{equation}

For stochastic sign descent, $s=\sign(g)$, we have $\mathbf{E}[s_i] = (2\rho_i - 1) \sign(\nabla \L_i )$ and thus $\nabla \L^T \mathbf{E}[s] = \sum_{i=1}^d \nabla\L_i \mathbf{E}[s_i] = \sum_i (2\rho_i-1) \vert\nabla \L_i \vert$.
Regarding the denominator, it is
\begin{equation}
\begin{split}
s^T Q s & \leq \left| \sum_{i=1}^d q_{ij} s_i s_j \right| \leq \sum_{i=1}^d \vert q_{ij} \vert \vert s_i\vert \vert s_j\vert \\
&= \sum_{i=1}^d \vert q_{ij} \vert,
\end{split}
\end{equation}
since $\vert s_i\vert = 1$.
Further, by definition of $p_\text{diag}(Q)$, we have $\sum_{i=1}^d \vert q_{ij} \vert = p_\text{diag}(Q)^{-1} \sum_{i=1}^d \vert q_{ii}\vert $.
Since $Q$ is positive definite, its diagonal elements are positive, such that $\sum_{i=1}^d \vert q_{ii}\vert = \sum_{i=1}^d q_{ii} =\sum_{i=1}^d \lambda_i$.
Plugging everything into Eq.~\eqref{eq:improvement_def} yields
\begin{equation}
\mathcal{I}_\textsc{ssd}  \geq \frac{1}{2} \frac{\left( \sum_{i=1}^d (2\rho_i - 1) \vert \nabla\L(\theta)_i\vert \right)^2}{\sum_{i=1}^d \lambda_i} p_\text{diag}(Q).
\end{equation}

\subsubsection{Properties of $p_\text{diag}(Q)$}
By writing $Q=\sum_k \lambda_k v_k v_k^T$ in its eigendecomposition with orthonormal eigenvectors $v_k\in\mathbb{R}^d$, we find
\begin{equation}
\begin{split}
\sum_{i,j}\vert q_{ij}\vert & = \sum_{i,j} \left| \sum_k \lambda_k v_{k,i} v_{k,j} \right| \leq \sum_{i,j} \sum_k \lambda_k \vert v_{k,i} v_{k,j} \vert \\
& = \sum_k \lambda_k \left( \sum_i \vert v_{k,i}\vert \right) \left( \sum_j \vert v_{k,j}\vert \right)\\
& \leq \sum_k \lambda_k \Vert v_k\Vert_1^2.
\end{split}
\end{equation}
As mentioned before, $\sum_i\vert q_{ii}\vert = \sum_i \lambda_i$.
Hence,
\begin{equation}
\label{eq:p_diag_with_eigenvectors}
p_\text{diag}(Q) = \frac{\sum_i\vert q_{ii}\vert}{\sum_{i,j}\vert q_{ij}\vert} = \frac{\sum_i \lambda_i}{\sum_i \lambda_i \Vert v_i\Vert_1^2}.
\end{equation}
As we have already seen, the best case arises if the eigenvectors are axis-aligned (diagonal $Q$), resulting in $\Vert v_i\Vert_1 = \Vert v_i\Vert_2 = 1$.

A worst case bound originates from the (tight) upper bound $\Vert w \Vert_1 \leq \sqrt{d} \Vert w\Vert_2$ for any $w\in\mathbb{R}^d$, which results in 
\begin{equation}
p_\text{diag}(Q) \geq \frac{1}{d}.
\end{equation}

We can get a rough intuition for the average case from the following consideration:
For a $d$-dimensional random vector $w\sim\mathcal{N}(0, I)$, which corresponds to a random orientation, we have
\begin{equation}
\mathbf{E}[\Vert w\Vert_2] \approx \sqrt{d},\quad \mathbf{E}[\Vert w\Vert_1] = d \sqrt{2/\pi}.
\end{equation}
As a rough approximation, we can thus assume that a randomly-oriented vector will satisfy $\Vert w\Vert_1 \approx \sqrt{2d/\pi} \Vert w\Vert_2$.
Plugging that in for the eigenvectors of $Q$ in Eq.~\eqref{eq:p_diag_with_eigenvectors} yields an approximate average case value of
\begin{equation}
p_\text{diag}(Q) \approx \frac{\pi}{2d} \approx \frac{1.57}{d}.
\end{equation}

\subsection{Variance Adaptation Factors}

\label{appendix_va_factors}

\begin{proof}[Proof of Lemma \ref{lemma:optimal_va_factors}]
Using $\mathbf{E}[\hat{p}_i]=p_i$ and $\mathbf{E}[\hat{p}_i^2]=p_i^2 + \sigma_i^2$, we get
\begin{equation}
\begin{split}
\mathbf{E} [\Vert \gamma \odot \hat{p} - p\Vert_2^2] & = \sum_{i=1}^d \mathbf{E}[(\gamma_i \hat{p}_i - p_i)^2] \\
& = \sum_{i=1}^d \left( \gamma_i^2\mathbf{E}[\hat{p}_i^2] -2 \gamma_i p_i \mathbf{E}[\hat{p}_i] + p_i^2 \right) \\
& = \sum_{i=1}^d \left( \gamma_i^2(p_i^2 + \sigma_i^2) - 2\gamma_i p_i^2 + p_i^2 \right).
\end{split}
\end{equation}
Setting the derivative w.r.t. $\gamma_i$ to zero, we find the optimal choice
\begin{equation}
\gamma_i = \frac{p_i^2}{p_i^2 + \sigma_i^2}.
\end{equation}
For the second part, using $\mathbf{E}[\sign(\hat{p}_i)] = (2\rho_i -1) \sign(p_i)$ and $\sign(\cdot)^2=1$, we get
\begin{equation}
\begin{split}
& \mathbf{E}[\Vert \gamma\odot\sign(\hat{p}) - \sign(p) \Vert_2^2] \\
& = \sum_{i=1}^d \mathbf{E} \left[ (\gamma_i \sign(\hat{p}_i) - \sign(p_i))^2 \right] \\
& = \sum_{i=1}^d \left(\gamma_i^2 - 2\gamma_i \sign(p_i)\mathbf{E}[\sign(\hat{p}_i)] + 1\right) \\
& = \sum_{i=1}^d \left( \gamma_i^2 - 2\gamma_i (2\rho_i-1)  + 1\right)
\end{split}
\end{equation}
and easily find the optimal choice
\begin{equation}
\gamma_i = 2\rho_i - 1.
\end{equation}
by setting the derivative to zero.
\end{proof}

\subsection{Convergence of Idealized SVAG}

\label{appendix_svag_convergence}

We prove the convergence results for idealized variance-adapted stochastic gradient descent (Theorem \ref{theorem:convergence_svag}).
The stochastic optimizer generates a discrete stochastic process $\{\theta_t\}_{t\in\mathbb{N}_0}$.
We denote as $\mathbf{E}_t[\cdot] = \mathbf{E}[\cdot \vert \theta_t]$ the conditional expectation given a realization of that process up to time step $t$.
Recall that $\mathbf{E}[\mathbf{E}_t[\cdot]]=\mathbf{E}[\cdot]$.

We first show the following Lemma.

\begin{lemma}
\label{lemma:bounds_gradient_norm_suboptimaliy}
Let $f\colon\mathbb{R}^d\to\mathbb{R}$ be $\mu$-strongly convex and $L$-smooth. Denote as $\theta_\ast := \argmin_{\theta\in\mathbb{R}^d} f(\theta)$ the unique minimizer and $f_\ast=f(\theta_\ast)$. Then, for any $\theta\in\mathbb{R}^d$,
\begin{equation}
\frac{2L^2}{\mu} (f(\theta) - f_\ast) \geq \Vert \nabla f(\theta) \Vert^2 \geq 2\mu (f(\theta)-f_\ast).
\end{equation}
\end{lemma}
\begin{proof}
Regarding the first inequality, we use $\nabla f(\theta_\ast)=0$ and the Lipschitz continuity of $\nabla f(\cdot)$ to get $\Vert \nabla f(\theta)\Vert^2 = \Vert \nabla f(\theta) - \nabla f(\theta_\ast)\Vert^2 \leq L^2 \Vert \theta - \theta_\ast\Vert^2$.
Using strong convexity, we have $f(\theta) \geq f_\ast + \nabla f(\theta_\ast)^T(\theta -\theta_\ast) + (\mu/2) \Vert \theta-\theta_\ast\Vert^2 = f_\ast + (\mu/2) \Vert \theta - \theta_\ast\Vert^2$.
Plugging the two inequalities together yields the desired inequality.

The second inequality arises from strong convexity, by minimizing both sides of
\begin{equation}
f(\theta^\prime) \geq f(\theta) + \nabla f(\theta)^T(\theta^\prime - \theta) + \frac{\mu}{2} \Vert \theta^\prime - \theta\Vert^2
\end{equation}
w.r.t. $\theta^\prime$.
The left-hand side obviously has minimal value $f_\ast$.
For the right-hand side, we set its derivative, $\nabla f(\theta) + \mu (\theta^\prime - \theta)$, to zero to find the minimizer $\theta^\prime=\theta-\nabla f(\theta)/\mu$.
Plugging that back in yields the minimal value $f(\theta) - \Vert \nabla f(\theta) \Vert/(2\mu)$.
\end{proof}

\begin{proof}[Proof of Theorem \ref{theorem:convergence_svag}]
Using the Lipschitz continuity of $\nabla f$, we can bound $f(\theta + \Delta\theta) \leq f(\theta) + \nabla f(\theta)^T \Delta\theta + \frac{L}{2} \Vert\Delta\theta\Vert^2$. Hence,
\begin{equation}
\begin{aligned}
& \mathbf{E}_t[f_{t+1}]\\
\leq &\, f_t - \alpha \mathbf{E}_t[ \nabla f_t^T(\gamma_t \odot g_t)] + \frac{L\alpha^2}{2} \mathbf{E}_t[\Vert \gamma_t \odot g_t\Vert^2 ]\\
= & \,f_t - \frac{1}{L} \sum_{i=1}^d \gamma_{t,i} \nabla f_{t,i} \mathbf{E}_t[g_{t,i}] + \frac{1}{2L} \sum_{i=1}^d \gamma_{t,i}^2 \mathbf{E}_t[g_{t,i}^2] \\
=  &\,f_t - \frac{1}{L} \sum_{i=1}^d \gamma_{t,i} \nabla f_{t,i}^2 + \frac{1}{2L} \sum_{i=1}^d \gamma_{t,i}^2 (\nabla f_{t,i}^2 + \sigma_{t,i}^2).
\end{aligned}
\end{equation}
Plugging in the definition $\gamma_{t,i} = \nabla f_{t,i}^2 / (\nabla f_{t,i}^2 + \sigma_{t,i}^2)$ and simplifying, we get
\begin{equation}
\label{eq:expected_decrease_isvag}
\mathbf{E}_t[f_{t+1}] \leq f_t - \frac{1}{2L} \sum_{i=1}^d \frac{\nabla f_{t,i}^4}{\nabla f_{t,i}^2 + \sigma_{t,i}^2}.
\end{equation}
This shows that $\mathbf{E}_t[f_{t+1}] \leq f_t$.
Defining $e_t:= f_t - f_\ast$, this implies
\begin{equation}
\mathbf{E}[e_{t+1}] = \mathbf{E}[\mathbf{E}_t[e_{t+1}]] \leq \mathbf{E}[e_t]
\end{equation}
and consequently, by iterating backwards, $\mathbf{E}[e_t] \leq \mathbf{E}[e_0] = e_0$ for all $t$.
Next, using the discrete version of Jensen's inequality\footnote{
Jensen's inequality states that, for a real convex function $\phi$, numbers $x_i\in\mathbb{R}$, and positive weights $a_i \in\mathbb{R}_+$ with $\sum_i a_i=1$, we have $\sum_i a_i \phi(x_i) \geq \phi\left(\sum_i a_i x_i\right)$.
We apply it here to the convex function $\phi(x)=1/x, x>0,$ with $x_i:=\frac{\nabla f_{t,i}^2 + \sigma_{t,i}^2}{\nabla f_{t,i}^2}$ and $a_i:=\frac{\nabla f_{t,i}^2}{\Vert \nabla f_t\Vert^2}$.} we find
\begin{equation}
\label{eq:expected_decrease_after_jensen}
\sum_{i=1}^d \frac{\nabla f_{t,i}^4}{\nabla f_{t,i}^2 + \sigma_{t,i}^2} \geq \frac{\Vert \nabla f_t\Vert^4}{\Vert \nabla f_t\Vert^2 + \sum_{i=1}^d \sigma_{t,i}^2}.
\end{equation}
Using the assumption $\sum_{i=1}^d \sigma_{t,i}^2 \leq c_v \vert \nabla f_t\Vert^2 + M_v$ in the denominator, we obtain
\begin{equation}
\label{eq:expected_decrease_after_variance_assumption}
\frac{\Vert \nabla f_t\Vert^4}{\Vert \nabla f_t\Vert^2 + \sum_{i=1}^d \sigma_{t,i}^2} \geq 
\frac{\Vert \nabla f_t\Vert^4}{(1+c_v)\Vert \nabla f_t\Vert^2 + M_v}.
\end{equation}
Using Lemma \ref{lemma:bounds_gradient_norm_suboptimaliy}, we have
\begin{equation}
\frac{2L^2}{\mu} e_t \geq \Vert \nabla f_t\Vert^2 \geq 2\mu e_t
\end{equation}
and can further bound
\begin{equation}
\label{eq:expected_decrease_in_terms_of_suboptimality}
\begin{split}
\frac{\Vert \nabla f_t\Vert^4}{(1+c_v)\Vert \nabla f_t\Vert^2 + M_v} & \geq \frac{4\mu^2 e_t^2}{\frac{2(1+c_v)L^2}{\mu} e_t + M_v} \\
& =: \frac{c_1 e_t^2}{c_2 e_t + c_3},
\end{split}
\end{equation}
where the last equality defines the (positive) constants $c_1, c_2$ and $c_3$.
Combining Eqs.~\eqref{eq:expected_decrease_after_jensen}, \eqref{eq:expected_decrease_after_variance_assumption} and \eqref{eq:expected_decrease_in_terms_of_suboptimality}, inserting in \eqref{eq:expected_decrease_isvag}, and subtracting $f_\ast$ from both sides, we obtain
\begin{equation}
\mathbf{E}_t[e_{t+1}] \leq e_t - \frac{1}{2L} \frac{c_1 e_t^2}{c_2 e_t + c_3},
\end{equation}
and, consequently, by taking expectations on both sides,
\begin{equation}
\begin{split}
\mathbf{E}[e_{t+1}]  & \leq \mathbf{E}[e_t] - \frac{1}{2L} \mathbf{E}\left[ \frac{c_1 e_t^2}{c_2 e_t + c_3} \right] \\
& \leq \mathbf{E}[e_t] - \frac{1}{2L} \frac{c_1 \mathbf{E}[e_t]^2}{c_2 \mathbf{E}[e_t] + c_3}
\end{split}
\end{equation}
where the last step is due to Jensen's inequality applied to the convex function $\phi(x) = \frac{c_1 x^2}{c_2 x + c_3}$.
Using $\mathbf{E}[e_t]\leq e_0$ in the denominator and introducing the shorthand $\bar{e}_t := \mathbf{E}[e_t]$, we get
\begin{equation}
\bar{e}_{t+1} \leq  \bar{e}_t - c \bar{e}_t^2 = \bar{e}_t(1-c\bar{e}_t),
\end{equation}
with $c := c_1/(2L(c_2e_0 + c_3))>0$.
To conclude the proof, we will show that this implies $\bar{e}_t\in\mathcal{O}(\frac{1}{t})$.
Without loss of generality, we assume $\bar{e}_{t+1}>0$ and obtain
\begin{equation}
\begin{split}
\bar{e}_{t+1}^{-1} & \geq \bar{e}_t^{-1}\left( 1- c \bar{e}_t \right)^{-1} \geq \bar{e}_t^{-1}\left( 1 + c\bar{e}_t \right) \\
& = \bar{e}_t^{-1} + c,
\end{split}
\end{equation}
where the second step is due to the simple fact that $(1-x)^{-1}\geq (1+x)$ for any $x\in [0, 1)$.
Summing this inequality over $t=0,\dotsc,T-1$ yields $\bar{e}_T^{-1} \geq e_0^{-1} + Tc$ and, thus,
\begin{equation}
T \bar{e}_T \leq \left( \frac{1}{Te_0} + c \right)^{-1} \overset{T\rightarrow\infty}{\longrightarrow} \frac{1}{c} < \infty,
\end{equation}
which shows that $\bar{e}_t\in\mathcal{O}(\frac{1}{t})$.
\end{proof}

\subsection{Gradient Variance Estimates via Moving Averages}
\label{details_moving_average_estimates}

We proof Eq.~\eqref{eq:expectation_of_m_squared}.
Iterating the recursive formula for $\tilde{m}_t$ backwards, we get
\begin{equation}
\label{eq:iterating_mt_backwards}
\begin{split}
m_t = \sum_{s=0}^t \underbrace{\frac{1-\beta_1}{1-\beta_1^{t+1}} \beta_1^{t-s}}_{=:c(\beta_1, t, s)} g_s,
\end{split}
\end{equation}
with coefficients $c(\beta_1, t, s)$ summing to one by the geometric sum formula, making $m_t$ a convex combination of stochastic gradients.
Likewise, $v_t = \sum_{s=0}^t c(\beta_2, t, s) g_s^2$ is a convex combination of squared stochastic gradients.
Hence,
\begin{equation}
\begin{split}
\mathbf{E}[m_{t,i}] & = \sum c(\beta, t, s) \mathbf{E}[g_{s,i}],\\
\mathbf{E}[v_{t,i}] & = \sum c(\beta, t, s) \mathbf{E}[g_{s,i}^2].
\end{split}
\end{equation}
Assumption \ref{assumption:iid_grads} thus necessarily implies $\mathbf{E}[g_{s,i}]\approx \nabla \L_{t,i}$ and $\mathbf{E}[g_{s,i}^2] \approx \nabla \L_{t,i}^2 + \sigma_{t,i}^2$.
(This will of course be utterly wrong for gradient observations that are far in the past, but these won't contribute significantly to the moving average.)
It follows that
\begin{equation}
\begin{split}
\mathbf{E}[m_{t,i}^2] & = \mathbf{E}[m_{t,i}]^2 + \mathbf{var}[m_{t,i}]\\
& = \nabla \L_{t, i}^2 + \sum_{s=0}^{t} c(\beta, t, s)^2 \, \mathbf{var}[g_{s,i}] \\
& = \nabla \L_{t,i}^2 + \sigma_{t,i}^2 \sum_{s=0}^t c(\beta, t, s)^2,
\end{split}
\end{equation}
where the second step is due to the fact that $g_s$ and $g_{s^\prime}$ are stochastically independent for $s\neq s^\prime$.
The last term evaluates to
\begin{equation}
\begin{split}
\rho(\beta, t) & := \sum_{s=0}^t c(\beta, t, s)^2  = \sum_{s=0}^t \left( \frac{1-\beta}{1-\beta^{t+1}} \beta^{t-s} \right)^2\\
& = \frac{(1-\beta)^2}{(1-\beta^{t+1})^2} \sum_{k=0}^t (\beta^2)^k \\
& = \frac{(1-\beta)^2}{(1-\beta^{t+1})^2} \, \frac{1-(\beta^2)^{t+1}}{1-\beta^2} \\
& = \frac{(1-\beta)(1-\beta)}{(1-\beta^{t+1})(1-\beta^{t+1})} \, \frac{(1-\beta^{t+1})(1+\beta^{t+1})}{(1-\beta)(1+\beta)}\\
& = \frac{(1-\beta)(1+\beta^{t+1})}{(1+\beta)(1-\beta^{t+1})},
\end{split}
\end{equation}
where the fourth step is another application of the geometric sum formula, and the fifth step uses $1-x^2 = (1-x)(1+x)$.
Note that
\begin{equation}
\rho(\beta, t) \rightarrow \frac{1-\beta}{1+\beta} \quad (t\rightarrow\infty),
\end{equation}
such that $\rho(\beta, t)$ is uniquely defined by $\beta$ in the long term.

As an interesting side note, the division by $1-\rho(\beta, t)$ in Eq.~\eqref{eq:bias_corrected_ema_estimate} is the  analogon to Bessel's correction (the use of $n-1$ instead of $n$ in the classical sample variance) for the case where we use moving averages instead of arithmetic means.

\subsection{Connection to Generalization}

\label{details_generalization}

\begin{proof}[Proof of Lemma \ref{lemma:wilson_lemma_extended_to_sd}]
Like in the proof of Lemma 3.1 in \citet{Wilson2017}, we inductively show that $\theta_t = \lambda_t \sign(X^T y)$ with a scalar $\lambda_t$.
This trivially holds for $\theta_0=0$.
Assume that the assertion holds for all $s\leq t$.
Then
\begin{equation}
\label{eq:least_squares_classification_gradient}
\begin{split}
\nabla R(\theta_t) & = \frac{1}{n} X^T(X\theta_t - y) \\
& = \frac{1}{n} X^T(\lambda_t X \sign(X^T y) - y)\\
& = \frac{1}{n} X^T (\lambda_t cy -y) = \frac{1}{n} (\lambda_t c -1) X^T y,
\end{split}
\end{equation}
where the first step is the gradient of the objective (Eq.~\ref{eq:least_squares_classification}), the second step uses the inductive assumption, and the third step uses the assumption $X\sign(X^Ty)=cy$.
Now, plugging Eq.~\eqref{eq:least_squares_classification_gradient} into the update rule, we find
\begin{equation}
\begin{split}
\theta_{t+1} & = \theta_t - \alpha \sign(\nabla R(\theta_t)) \\
& = \lambda_t \sign(X^Ty) - \alpha \sign((\lambda_t c -1) X^T y) \\
& = (\lambda_t - \alpha \sign(\lambda_tc-1)) \sign(X^Ty).
\end{split}
\end{equation}
Hence, the assertion holds for $t+1$.
\end{proof}

\section{Alternative Methods}

\label{appendix_alternative_methods}

\subsection{SVAG}
\label{appendix_svag}

\textsc{m-svag} applies variance adaptation to the update direction $m_t$, resulting in the variance adaptation factors Eq.~\ref{eq:estimated_va_factors_msvag}.
We can also update in direction $g_t$ and choose the appropriate estimated variance adaptation factors, resulting in an implementation of \textsc{svag} without momentum.
We have already derived the necessary variance adaptation factors en route to those for the momentum variant, see Eq.~\eqref{eq:estimated_va_factors_grad} in \textsection\ref{estimating_svag_factors}.
Pseudo-code is provided in Alg.~\ref{alg:svag_ema}.
It differs from \textsc{m-svag} only in the last two lines.

\begin{algorithm}
\footnotesize
\caption{\textsc{svag}}
\label{alg:svag_ema}
\begin{algorithmic}
\STATE {\bfseries Input:} $\theta_0\in\mathbb{R}^d$, $\alpha>0$, $\beta \in [0,1]$, $T\in\mathbb{N}$
\STATE Initialize $\theta\gets \theta_0$, $\tilde{m}\gets 0$, $\tilde{v}\gets 0$
\FOR{$t=0,\dotsc, T-1$}
  \STATE $\tilde{m}\gets \beta \tilde{m} + (1-\beta) g(\theta)$, \quad $\tilde{v}\gets \beta \tilde{v} + (1-\beta)g(\theta)^2$
  \vspace{1pt}
  \STATE $m \gets (1-\beta^{t+1})^{-1} \tilde{m}$, \quad $v \gets (1-\beta^{t+1})^{-1}\tilde{v}$
  \vspace{1pt}
  \STATE $s\gets (1-\rho(\beta, t))^{-1} (v-m^2)$ 
  \vspace{1pt}
  \STATE $\gamma \gets m^2/(m^2 + s)$
  \vspace{1pt}
  \STATE $\theta \gets \theta - \alpha (\gamma\odot g)$
\ENDFOR
\end{algorithmic}
\end{algorithm}

\subsection{Variants of ADAM}
\label{appendix_myadam}

This paper interpreted \textsc{adam} as variance-adapted \textsc{m-ssd}.
The experiments in the main paper used a standard implementation of \textsc{adam} as described by \citet{Kingma2014}.
However, in the derivation of our implementation of \textsc{m-svag}, we have made multiple adjustments regarding the estimation of variance adaptation factors which correspondingly apply to the sign case.
Specifically, this concerns:
\begin{itemize}
\item The use of the same moving average constant for the first and second moment ($\beta_1=\beta_2=\beta$).
\item The bias correction in the gradient variance estimate, see Eq.~\eqref{eq:bias_corrected_ema_estimate}.
\item The adjustment of the variance adaptation factors for the momentum case, see \textsection\ref{incorporating_momentum}.
\item The omission of a constant offset $\varepsilon$ in the denominator.
\end{itemize}
Applying these adjustment to the sign case gives rise to a variant of the original \textsc{adam} algorithm, which we will refer to as \textsc{adam*}.
Pseudo-code is provided in Alg.~\ref{alg:myadam}. 
Note that we use the variance adaptation factors $(1+\eta)^{-1/2}$ and \emph{not} the optimal ones derived in \textsection\ref{va_for_sign}, which would under the Gaussian assumption be $\erf[(\sqrt{2} \eta)^{-1}]$.
We initially experimented with both variants and found them to perform almost identically, which is not surprising given how similar the two are (see Fig.~\ref{fig:va_factors}).
We thus stuck with the first option for direct correspondence with the original \textsc{adam} and to avoid the cumbersome error function.

In analogy to \textsc{svag} versus \textsc{m-svag}, we could also define a variance-adapted version stochastic sign descent \emph{without} momentum, i.e., using the base update direction $\sign(g_t)$.
We did not explore this further in this work.

\begin{algorithm}
\footnotesize
\caption{\textsc{adam*}}
\label{alg:myadam}
\begin{algorithmic}
\STATE {\bfseries Input:} $\theta_0\in\mathbb{R}^d$, $\alpha>0$, $\beta \in [0,1]$, $T\in\mathbb{N}$
\STATE Initialize $\theta\gets \theta_0$, $\tilde{m}\gets 0$, $\tilde{v}\gets 0$
\FOR{$t=0,\dotsc, T-1$}
  \STATE $\tilde{m}\gets \beta \tilde{m} + (1-\beta) g(\theta)$, \quad $\tilde{v}\gets \beta \tilde{v} + (1-\beta)g(\theta)^2$
  \vspace{1pt}
  \STATE $m \gets (1-\beta^{t+1})^{-1} \tilde{m}$, \quad $v \gets (1-\beta^{t+1})^{-1}\tilde{v}$
  \vspace{1pt}
  \STATE $s\gets (1-\rho(\beta, t))^{-1} (v-m^2)$ 
  \vspace{1pt}
  \STATE $\gamma \gets \sqrt{ m^2/(m^2 + \rho(\beta, t)s)}$
  \vspace{1pt}
  \STATE $\theta \gets \theta - \alpha (\gamma\odot \sign(m))$
\ENDFOR
\end{algorithmic}
\end{algorithm}


\subsection{Experiments}

We tested \textsc{svag} as well as \textsc{adam*} with and without momentum on the problems (P2) and (P3) from the main paper.
Results are shown in Figure \ref{fig:results_alternatives}.

We observe that \textsc{svag} performs better than \textsc{m-svag} on (P2).
On (P3), it makes faster initial progress but later plateaus, leading to slightly worse outcomes in both training loss and test accuracy.
\textsc{svag} is a viable alternative.
In future work, it will be interesting to apply \textsc{svag} to problems where \textsc{sgd} outperforms \textsc{m-sgd}.

Next, we compare \textsc{adam*} to the original \textsc{adam} algorithm.
In the \textsc{cifar-100} example (P3) the two methods are on par.
On (P2), \textsc{adam} is marginally faster in the early stages of the the optimization process.
\textsc{adam*} quickly catches up and reaches lower minimal training loss values.
We conclude that the adjustments to the variance adaptation factors derived in \textsection\ref{practical_implementation} do have a positive effect.

\begin{figure}[t]
\centering
\includegraphics[scale=0.98]{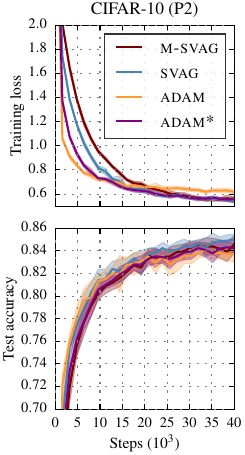}
\includegraphics[scale=0.98]{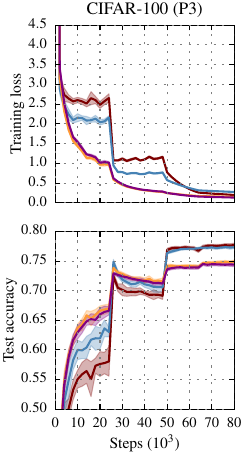}
\caption{Experimental results for \textsc{svag} and \textsc{adam*}. The plot is set-up like Fig.~\ref{fig:results}.}
\label{fig:results_alternatives}
\end{figure}

\section{Mini-Batch Gradient Variance Estimates}

\label{appendix_mini_batch_variance_estimate}

In the main text, we have discussed estimation of gradient variances via moving averages of the past gradient observations.
An alternative gradient variance estimate can be obtained locally, within a single mini-batch.
The individual gradients $\nabla \ell(\theta; x_k)$ in a mini-batch are iid random variables and $\mathbf{var}[g(\theta)] = \vert \B\vert^{-1} \mathbf{var}_{k\sim\mathcal{U}([M])}[\nabla \ell(\theta; x_k)]$.
We can thus estimate $g(\theta)$'s variances by computing the sample variance of the $\{\nabla \ell(\theta; x_k)\}_{k\in\mathcal{B}}$, then scaling by $\vert\B\vert^{-1}$,
\begin{equation}
\label{eq:sample_variance}
\hat{s}^\text{mb}(\theta) = \frac{1}{\vert\B\vert} \left( \frac{1}{\vert\B\vert-1} \sum_{k\in\B} \nabla \ell(\theta; x_k)^2 - g(\theta)^2 \right).
\end{equation}
Several recent papers~\citep{Mahsereci2015, Balles2017} have used this variance estimate for other aspects of stochastic optimization.
In contrast to the moving average-based estimators, this is an unbiased estimate of the \emph{local} gradient variance.
The (non-trivial) implementation of this estimator for neural networks is described in \citet{Balles2017workshop}.

\subsection{M-SVAG with Mini-Batch Estimates}

We explored a variant of \textsc{m-svag} which use mini-batch gradient variance estimates.
The local variance estimation allows for a theoretically more pleasing treatment of the variance of the update direction $m_t$.
Starting from the formulation of $m_t$ in Eq.~\eqref{eq:iterating_mt_backwards} and considering that $g_s$ and $g_{s^\prime}$ are stochastically independent for $s\neq s^\prime$, we have
\begin{equation}
\mathbf{var}[m_t] = \sum_{s=0}^t \left( \frac{1-\beta}{1-\beta^{t+1}} \beta^{t-s} \right)^2 \mathbf{var}[g_s].
\end{equation}
Given that we now have access to a true, local, unbiased estimate of $\mathbf{var}[g_s]$, we can estimate $\mathbf{var}[m_t]$ by
\begin{equation}
\label{eq:minibatch_variance_estimate_of_mt}
\bar{s}_t := \sum_{s=0}^t \left( \frac{1-\beta}{1-\beta^{t+1}} \beta^{t-s} \right)^2 \hat{s}^\text{mb}(\theta_s).
\end{equation}
It turns out that we can track this quantity with another exponential moving average: It is $\bar{s}_t = \rho(\beta, t) r_t$ with
\begin{gather}
\label{eq:minibatch_variance_estimate_of_mt_as_moving_average}
\tilde{r}_t = \beta^2 \tilde{r}_{t-1} + (1-\beta^2) \hat{s}^\text{mb}_t, \quad r_t = \frac{\tilde{r}_t}{1-(\beta^2)^{t+1}}.
\end{gather}
This can be shown by iterating Eq.~\eqref{eq:minibatch_variance_estimate_of_mt_as_moving_average} backwards and comparing coefficients with Eq.~\eqref{eq:minibatch_variance_estimate_of_mt}.
The resulting mini-batch variant of \textsc{m-svag} is presented in Algorithm \ref{alg:msvag_mb}.
\begin{algorithm}[b]
\footnotesize
\caption{\textsc{m-svag} with mini-batch variance estimate}
\label{alg:msvag_mb}
\begin{algorithmic}
\STATE {\bfseries Input:} $\theta_0\in\mathbb{R}^d$, $\alpha>0$, $\beta \in [0,1]$, $T\in\mathbb{N}$
\STATE Initialize $\theta\gets \theta_0$, $\tilde{m}\gets 0$, $\tilde{r}\gets 0$
\FOR{$t=0,\dotsc, T-1$}
  \STATE Compute mini-batch gradient $g(\theta)$ and variance $\hat{s}^\text{mb}(\theta)$
  \STATE $\tilde{m}\gets \beta \tilde{m} + (1-\beta) g(\theta)$, \quad $\tilde{r}\gets \beta^2 \tilde{r} + (1-\beta^2) \hat{s}^\text{mb}(\theta)$
  \STATE $m \gets (1-\beta^{t+1})^{-1} \tilde{m}$, \quad $r \gets (1-\beta^{2(t+1)})^{-1}\tilde{r}$
  \STATE $\gamma \gets m^2/(m^2 + \rho(\beta, t) r)$
  \STATE $\theta \gets \theta - \alpha (\gamma\odot m)$
\ENDFOR
\end{algorithmic}
\end{algorithm}

Note that mini-batch gradient variance estimates could likewise be used for the alternative methods discussed in \textsection\ref{appendix_alternative_methods}.
We do not explore this further in this paper.

\subsection{Experiments}

We tested the mini-batch variant of \textsc{m-svag} on the problems (P1) and (P2) from the main text and compared it to the moving average version.
Results are shown in Figure \ref{fig:results_msvagmb}.
The two algorithms have almost identical performance.

\begin{figure}[t]
\centering
\includegraphics[scale=0.98]{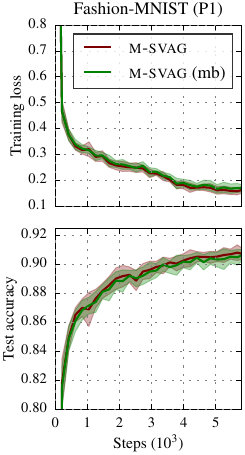}
\includegraphics[scale=0.98]{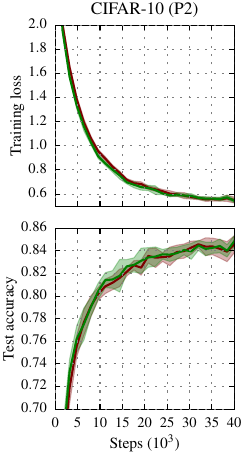}
\caption{Experimental results for the mini-batch variant of \textsc{m-svag} (marked ``mb'' in the legend).
The plot is set-up like Fig.~\ref{fig:results}.}
\label{fig:results_msvagmb}
\end{figure}